\documentclass{article}

\usepackage[preprint, nonatbib]{neurips_2021}

\usepackage[utf8]{inputenc} 
\usepackage[T1]{fontenc}    
\usepackage{hyperref}       
\usepackage{url}            
\usepackage{booktabs}       
\usepackage{amsfonts}       
\usepackage{nicefrac}       
\usepackage{microtype}      

\usepackage{amssymb}
\usepackage{physics}
\usepackage{graphicx}
\usepackage{wrapfig}
\usepackage{hyperref} 
\hypersetup{
    colorlinks,
    citecolor=blue,
    filecolor=black,
    linkcolor=red,
    urlcolor=blue
}

\newcommand{\qqc}{,\qquad}
\newcommand{\E}[1][]{\mathrm{E}_{#1}}
\newcommand{\Var}[1][]{\mathrm{Var}_{#1}}
\newcommand{\Cov}[1][]{\mathrm{Cov}_{#1}}
\newcommand{\Bias} {\mathrm{Bias}}

\newcommand{\vbx}{\vec{\mathbf{x}}}
\newcommand{\vby}{\vec{\mathbf{y}}}
\newcommand{\vbz}{\vec{\mathbf{z}}}
\newcommand{\vbf}{\vec{\mathbf{f}}}
\newcommand{\vbbeta}{\vec{\boldsymbol{\beta}}}
\newcommand{\vbeps}{\vec{\boldsymbol{\varepsilon}}}
\newcommand{\hbw}{\hat{\mathbf{w}}}
\newcommand{\hby}{\hat{\mathbf{y}}}
\newcommand{\hbx}{\hat{\mathbf{x}}}

\usepackage{xcolor}

\usepackage{amsthm}

\title{The Geometry of Over-parameterized Regression and Adversarial Perturbations}

\author{%
  Jason W.~Rocks\\
  Department of Physics\\
  Boston University\\
  Boston, MA 02215\\
  \texttt{jrocks@bu.edu} 
  \And
  Pankaj~Mehta\thanks{Corresponding author.}\\
  Department of Physics\\
  Faculty of Computing and Data Sciences\\
  Boston University\\
  Boston, MA 02215\\
  \texttt{pankajm@bu.edu} 
}

\begin{document}

\maketitle

\begin{abstract}
 Classical regression has a simple geometric description in terms of a projection of the training labels onto the column space of the design matrix. 
 However, for over-parameterized models  -- where the number of fit parameters is large enough to perfectly fit the training data -- this picture becomes uninformative.
Here, we present an alternative geometric interpretation of regression that applies to both under- and over-parameterized models. 
Unlike the classical picture which takes place in the space of training labels, our new picture resides in the space of input features.
This new feature-based perspective provides a natural geometric interpretation of the double-descent phenomenon in the context of bias and variance, 
explaining why it can occur even in the absence of label noise.
Furthermore, we show that adversarial perturbations -- small perturbations to the input features that result in large changes in label values --  are a generic feature of biased models, arising from the underlying geometry.
We demonstrate these ideas by analyzing three minimal models for over-parameterized linear least squares regression: without basis functions (input features equal model features) and with linear or nonlinear basis functions (two-layer neural networks with linear or nonlinear activation functions, respectively).
\end{abstract}

\section{Introduction}

Classical statistics has a long-standing practice of using geometry as a tool to understand the behavior of linear regression models. 
Ordinary least squares (OLS) regression has a particularly elegant interpretation: given a training data set consisting of input features $\vbx$ paired with labels (responses) $y$, the fitted values $\hat{y}$ for the labels can be calculated by orthogonally projecting them onto the column space of the design matrix ~\cite{Bishop2006} (see Sec.~\ref{sec:label_space} and Fig.~\ref{fig:data_space}). Geometric descriptions have also proven useful for understanding related methods such as instrumental variables analysis, where regression has a natural interpretation in terms of oblique projection~\cite{Bowden1990}.

In the classical setting, a model's performance on new data points (the generalization or test error) is well-understood in terms of the bias-variance trade-off:
while increasing a model's complexity (e.g., increasing the number of fit parameters) reduces bias (error due to the inability to fully express patterns hidden in the data),
it increases variance (error arising from over-sensitivity to non-general aspects of the training set like noise).
This trade-off is reflected in the test error in the form of a classical, ``U-shaped'' curve in which the test error first decreases with model complexity until it reaches a minimum before increasing dramatically as the model overfits the training data. Within this framework, optimal performance is achieved at intermediate model complexities which strike a balance between bias and variance~\cite{Abu-Mostafa2012, Hastie2016}.

However, modern supervised learning techniques like Deep Learning methods defy this dogma, achieving state-of-the-art performance using  ``over-parameterized'' models, where the number of fit parameters is large enough to perfectly fit the training data~\cite{Zhang2017}. In fact, increasing model complexity beyond the interpolation threshold -- the point at which the training error reaches zero -- actually results in a decrease in test error due to a drop in both bias and variance, giving rise to what is now commonly referred to as ``double-descent'' curves~\cite{Belkin2019, Loog2020} [see Fig.~\ref{fig:data_space}(a)]. Recently, the double-descent phenomenon has been shown to be a fundamental property of over-parameterized models ranging from modern neural networks to classical linear regression~\cite{Adlam2020, Advani2020, Ba2020, Barbier2019, Bartlett2020, Belkin2020, Bibas2019, Deng2019, DAscoli2020, DAscoli2020a, Derezinski2020, Gerace2020, Hastie2019, Jacot2020, Kini2020, Lampinen2019, Li2020, Liang2020, Liang2020b, Liao2020, Lin2020, Mitra2019, Mei2019, Muthukumar2019, Nakkiran2019, Nakkiran2020, Rocks2020, Xu2019, Yang2020}.

In this modern setting, the classical geometric picture suffers from two significant drawbacks.
First, the description is uninformative for over-parameterized models since  by definition, they always achieve zero training error and hence the relevant projector simply becomes the identity.
Second, the classical picture fails to provide any intuition about test error and the double-descent phenomenon.

In light of these inadequacies, here we present an alternative geometric picture of OLS that resides in the space of input features rather than the space of training labels. 
This picture applies to both under- and over-parameterized models and provides simple geometric explanations for the double-descent phenomenon and adversarial perturbations. 
We illustrate these ideas by considering three different choices of basis functions for OLS regression: (i) no basis functions (model features equal input features), (ii) random linear features (a two-layer neural network with a linear activation function and a random, fixed middle layer), and (iii) random nonlinear features (a two-layer neural network with an arbitrary nonlinear activation function and a random, fixed middle layer).

\paragraph{Summary of Results} We briefly summarize the contributions of this work:
\begin{itemize}
\item We show that when an OLS model creates a label prediction $\hat{y}$ for a data point $(y, \vbx)$, it also produces an internal representation of the input features $\hbx$. As a result, every data point is associated with a pair of estimates for both the dependent \textit{and} independent variables $(\hat{y}, \hbx)$.
\item We show that the internal representation of the input features can be expressed as ${\hbx = P_f\vbx}$, where the matrix operator $P_f$ first maps each data point onto the subspace  $\mathcal{F}_W$ of input features the model can express, and then onto the subspace $\mathcal{F}_X$ of input feature spanned by the training data.
\item We discuss how the choice of basis functions for the model features affects the geometric nature of the operator $P_f$ in the space of input features: (i) a lack of basis functions results in an orthogonal projection, (ii) a random linear basis results in an oblique projection, and (iii) a random nonlinear basis results in what we call a ``noisy'' oblique projection.
\item We show how model prediction accuracy can be characterized by the angles between directions in the subspaces $\mathcal{F}_W$ and $\mathcal{F}_X$ related by $P_f$. In particular, when one of these angles approaches $90^\circ$, a model loses its ability to express a direction in the training data, resulting in overfitting. We show how this behavior gives rise to the double-descent phenomenon even in the absence of label noise.
\item We study the geometry of adversarial perturbations, showing that perturbations to the input features of a data point decompose along ``adversarial directions'' (the complement of the kernel $P_f$) that result in large typical changes in the label prediction and ``invariant directions'' (the kernel $P_f$) that result in no change on average. We show that the subspace of invariant directions stems from model bias.
\end{itemize}

\section{Theoretical Setup}

In this work, we focus on the supervised learning task of OLS regression,
in which a linear model uses the relationships learned from a training data set to accurately predict the labels of new data points based on their input features.
We consider a training data set of $M$ data points, $\mathcal{D} = \{(y_a, \vbx_a)\}_{a=1}^M$, each consisting of a label $y$ and a vector of $N_f$ input features $\vbx$.
Each label is related to its corresponding input features via the relationship (teacher model)
\begin{equation}
y(\vbx) = y^*(\vbx) + \varepsilon,\label{eq:teacher}
\end{equation}
where $y^*(\vbx)$ is an unknown function and $\varepsilon$ is the label noise.
For convenience, we organize the vectors of input features in the training data set into the rows of a design matrix $X$ of size $M\times N_f$ and define the length $M$ vectors of training labels $\vby$ and label noise $\vbeps$.

Given an arbitrary data point $\vbx$, the linear model for the label predictions (student model) is then
\begin{equation}
\hat{y} = \vbz(\vbx)\cdot\hbw,\label{eq:student}
\end{equation}
where  $\hbw$ is a vector of $N_p$ fit parameters determined by the training set. 
The vector of model features $\vbz(\vbx)$ is a deterministic function that transforms the input features $\vbx$ into a vector of $N_p$ basis functions (referred to as ``hidden'' features in the context of a two-layer neural network).
In analogy to the observation matrix, we organize the vectors of model features evaluated on the input features of the training set,  
$\{\vbz(\vbx_a)\}_{a=1}^M$, into the rows of a model feature matrix $Z$ of size $M\times N_p$ and define the length $M$ vector of training label predictions $\hby$.

We uniquely determine the fit parameters $\hbw$ by requiring that they minimize the following loss function evaluated on the training set $\mathcal{D}$:
\begin{equation}
L(\hbw; \mathcal{D}) = \frac{1}{2}\norm{\Delta \vby}^2 + \frac{\lambda}{2}\norm{\hbw}^2,\label{eq:loss}
\end{equation}
where $\norm{\cdot}$ is the $L_2$ norm and  $\Delta \vby = \vby - \hby$ is the vector of residual training label errors.
The first term in the loss is the mean-squared label error and the second term imposes $L_2$ regularization with regularization parameter $\lambda$.
To simplify the conceptual picture presented in this work, we consider the ridge-less limit $\lambda\rightarrow 0$, allowing us to express the fit parameters that minimize Eq.~\eqref{eq:loss} and the resulting label predictions in Eq.~\eqref{eq:student}, respectively, as
\begin{equation}
\hbw = Z^+\vby\qqc \hat{y} = \vbz(\vbx)\cdot Z^+\vby, \label{eq:solution}
\end{equation}
where $^+$ denotes the Moore-Penrose inverse, or pseudoinverse. 
Later, we discuss the geometric interpretations of this solution for different sets of basis functions for the model features $\vbz(\vbx)$.

\section{Geometry of Label Space}\label{sec:label_space}

\begin{figure*}[t]
\centering
\includegraphics[width=\linewidth]{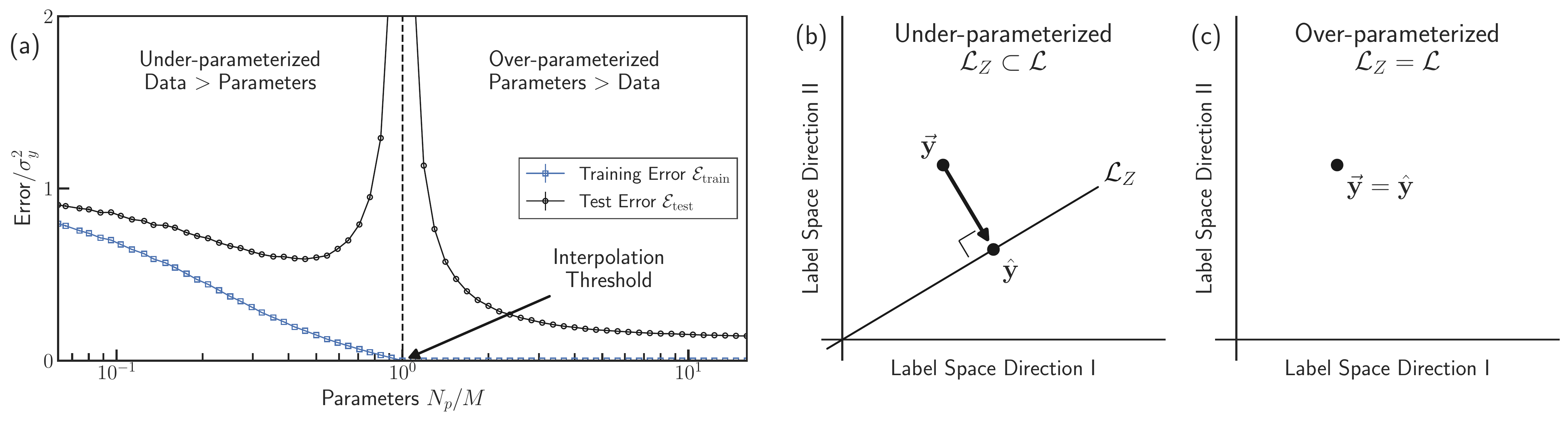}
\caption{\textbf{Double-descent phenomenon and corresponding geometric picture in training label space $\mathcal{L}$.} {\bf(a)} The training error (blue squares) and test error (black circles) for ordinary least square regression as a function of the ratio of fit parameters $N_p$ to training data points $M$. 
In the under-parameterized regime ($N_p < M$), increasing $N_p$ relative to $M$ decreases the training error until it reaches zero at the interpolation threshold (vertical dashed line).
At the same time, the test error follows a classical ``U-shaped'' curve.
In the over-parameterized regime ($N_p > M$), the test error decreases while the training error remains zero.
The $y$-axis is scaled by the variance of the training labels $\sigma_y^2$.
{\bf(b)} In the under-parameterized regime, the label predictions for the training set $\hby$ are found by orthogonally projecting the training labels $\vby$ onto the subspace $\mathcal{L}_Z\subset \mathcal{L}$ spanned by the rows of the matrix of model features evaluated on the training set $Z$.
{\bf(c)} In the over-parameterized regime, the geometric picture is trivial since the two spaces are identical, $\mathcal{L}_Z = \mathcal{L}$, with the orthogonal projection resulting in label predictions with zero error,  $\hby = \vby$. See Sec.~\ref{sec:dd} and SI for numerical details.}
\label{fig:data_space}
\vskip -1em
\end{figure*}

Before introducing the geometric interpretation of the OLS solution in the input feature space,
we first review the classical picture of this solution in the training label space.
First, we define the training label space $\mathcal{L}$ as the $M$-dimensional vector space spanned by all possible values of the vector of training labels $\vby$ 
(e.g., here we consider $\mathcal{L}=\mathbb{R}^M$ for continuous labels distributed over $\mathbb{R}$).
Based on Eq.~\eqref{eq:solution}, the vector of predicted labels takes the form
\begin{equation}
\hby = P_\ell\vby\qqc P_\ell = ZZ^+,
\end{equation}
where the $M\times M$ matrix $P_\ell$ represents an orthogonal projection (${P_\ell^2 = P_\ell = P_\ell^T}$) onto the subspace spanned by the training data via the model features, or the rows of $Z$ which we denote as $\mathcal{L}_Z$. 
In Fig.~\ref{fig:data_space}(b), we depict an example of this projection in the classical under-parameterized setting for a synthetic data set where we have represented $\mathcal{L}_Z$ as a line (see Sec.~\ref{sec:dd} and SI for details).

Most importantly, this picture provides a natural geometric form for the training error, 
\begin{align}
\mathcal{E}_{\mathrm{train}} &= \frac{1}{M} \norm{\Delta \vby}^2\qqc \Delta \vby = (I_M-P_\ell)\vby,\label{eq:training_err}
\end{align}
where $I_M$ is the $M\times M$ identity matrix.
From this, we see that the residual training label errors $\Delta \vby$ are nonzero when $\mathcal{L}_Z$ spans only a fraction of the possible directions in the training label space $\mathcal{L}$.

While this description provides intuition for the under-parameterized regime, its utility diminishes for over-parameterized models.
As the complexity (e.g., the number of fit parameters) of a model is increased, the training error decreases as the dimension of $\mathcal{L}_Z$ increases.
Once the number of parameters is large enough, the training error reaches zero at the interpolation threshold as shown in Fig.~\ref{fig:data_space}(a).
Past this threshold in the over-parameterized, $\mathcal{L}_Z$ spans all of $\mathcal{L}$ and  $P_\ell$ trivially becomes the identity operator $I_M$,
and as shown in Fig.~\ref{fig:data_space}(c), this leads to zero residual label error, $\vby=\hby$.
Furthermore, this geometric picture provides no intuition for the double-descent profile of the test error in Fig.~\ref{fig:data_space}(a) for neither the under- nor over-parameterized regimes.

\section{Geometry of Feature Space}\label{sec:feature}

\begin{figure*}[t]
\centering
\centerline{\includegraphics[width=\linewidth]{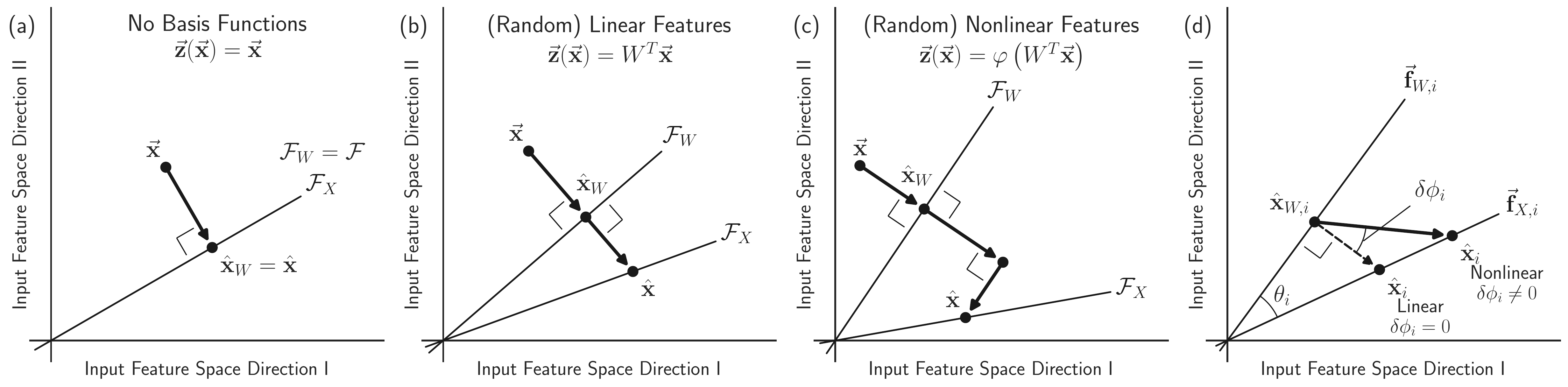}}
\caption{
{\bf Geometric picture in input feature space $\mathcal{F}$.}
{\bf (a)} \emph{No Basis Functions}.
To predict the label of a new data point, the input features $\vbx$ are orthogonally projected onto the subspace $\mathcal{F}_X$ spanned by the input features of the training data (row space of $X$) to obtain $\hbx$, 
the datapoint as ``seen'' from the perspective the model [see Eq.~\eqref{eq:yhat_decomp}]. 
{\bf (b) } \emph{(Random) Linear Features}. 
The input features $\vbx$ undergo an oblique projection to obtain $\hbx$ which depends on the subspace $\mathcal{F}_X$, along with $\mathcal{F}_W$, the subspace of input features that the model can express (column space of $W$).
To perform the oblique projection, the model first orthogonally projects onto $\mathcal{F}_W$ to obtain $\hbx_W$ and then continues orthogonally from $\mathcal{F}_W$ until it intersects $\mathcal{F}_X$ to obtain $\hbx$.
{\bf (c)} \emph{(Random) Nonlinear Features}.
The input features $\vbx$ undergo a ``noisy'' oblique projection to obtain $\hbx$,  which again depends on the two subspaces $\mathcal{F}_X$ and $\mathcal{F}_W$.
Similar to an oblique projection, $\hbx$ is first orthogonally projected onto $\mathcal{F}_W$.
However, unlike an oblique projection, the mapping from $\mathcal{F}_W$ to $\mathcal{F}_X$ has an extra component parallel to $\mathcal{F}_W$, introducing additional error.
{\bf (d)} Mapping between subspaces performed by each term in the singular value decomposition of $P_f$.
The subspace orientation angle $\theta_i$  measures the relative orientation between the left and right singular vectors, $\vbf_{W, i}$ and $\vbf_{X, i}$,
while the projection deviation angle $\delta \phi_i$ is  zero when the mapping is orthogonal to $\vbf_{W, i}$. See Sec.~\ref{sec:feature} for details.
}
\label{fig:feature_space}
\vskip -1em
\end{figure*}

Next, we introduce the geometric interpretation of the OLS solution in the input feature space $\mathcal{F}$. 
We define this space as the $N_f$-dimensional vector space spanned by all possible values the of input features $\vbx$ 
(e.g., here we consider $\mathcal{F}=\mathbb{R}^{N_f}$  for continuous input features distributed over $\mathbb{R}$). In the SI, we formalize all following results as theorems with corresponding proofs when appropriate.

\paragraph{Geometric Decomposition of Label Predictions}
To develop our geometric description, we first decompose the labels and model features into their linear and nonlinear components. 
For a given data point $(y, \vbx)$, we define the label decomposition
\begin{equation}
y(\vbx) = \vbx\cdot\vbbeta +  \delta y^*_{\mathrm{NL}}(\vbx) +  \varepsilon\qqc   \vbbeta \equiv \Sigma_{\vbx}^{-1}\Cov[\vbx]\qty[\vbx, y(\vbx)].\label{eq:ydecomp}
\end{equation}
The first term of $y(\vbx)$ captures the linear correlations between the input features $\vbx$ and the labels $y^*$, characterized by the vector of $N_f$ parameters $\vbbeta$.
The notation $\Cov[\vbx][\cdot, \cdot]$ represents the covariance evaluated with respect to the distribution of input features and we define ${\Sigma_{\vbx, ij}\equiv\Cov[\vbx][x_i, x_j]}$ as the covariance matrix of the input features (assumed to be invertible).
The second term $\delta y^*_{\mathrm{NL}}(\vbx)$ captures the remaining nonlinear component of $y^*$ [defined as ${\delta y^*_{\mathrm{NL}}(\vbx) \equiv y^*(\vbx) - \vbx\cdot\vbbeta}$] 
and is statistically independent of the linear term with respect to the input features (see SI for proof).

Next, we define the model feature decomposition
\begin{equation}
\vbz(\vbx) = W^T\vbx +  \delta \vbz_{\mathrm{NL}}(\vbx)\qqc  W \equiv  \Sigma_{\vbx}^{-1}\Cov[\vbx]\qty[\vbx, \vbz^T(\vbx)].\label{eq:zdecomp}
\end{equation}
Analogous to the label decomposition, the first term of $\vbz(\vbx)$ captures the linear correlations between the input features $\vbx$ and the model features $\vbz$, 
characterized by the $N_f\times N_p$ matrix $W$.
The second term $\delta \vbz_{\mathrm{NL}}(\vbx)$ captures the remaining nonlinear component of $\vbz$ [defined as ${\delta \vbz_{\mathrm{NL}}(\vbx) \equiv \vbz(\vbx) - W^T\vbx}$],
and is also statistically independent of the linear term with respect to the input features (see SI for proof).

Using the definitions above,
we decompose the solution for the predicted labels in Eq.~\eqref{eq:solution} for an arbitrary data point $\vbx$ as follows  (see SI for proof):
\begin{equation}
\begin{gathered}
\hat{y}(\vbx) = \hbx\cdot \vbbeta  +   \delta \hat{y}(\vbx)\qqc \hbx \equiv P_f\vbx \qqc P_f \equiv (WZ^+X)^T\\
\delta \hat{y} (\vbx)  \equiv  \delta \vbz_{\mathrm{NL}}^T(\vbx)  Z^+ \vby  + \vbx^TWZ^+\qty(\delta \vby^*_{\mathrm{NL}}+\vbeps),\label{eq:yhat_decomp}
\end{gathered}
\end{equation}
where we have defined the length $N_f$ vector $\hbx$ in terms of the $N_f\times N_f$ matrix operator $P_f$.
The quantity $\delta \hat{y} (\vbx)$ captures the elements of the labels or model features that the model interprets as ``noise'' and goes to zero in the absence of label noise and nonlinearities.

\paragraph{Internal Representation of Input Features} 
The vector $\hbx$ and the operator $P_f$ play central roles in our geometric description.
Notice that in the expressions above, $\hbx$ is the counterpart to the label prediction $\hat{y}$, so that for a given data point $(y, \vbx)$, the model produces a pair of estimates for both the dependent and independent variables $(\hat{y}, \hbx)$. 
For this reason, one can interpret $\hbx$ as the model's internal representation of the input features $\vbx$.
In other words, $\hbx$ is the data point as ``seen'' from the perspective of the model via the operator $P_f$,
which encodes the patterns a model has learned from the training set $X$ based on its model features as characterized by $W$.

\paragraph{Mapping Between Subspaces} 
In order to construct the estimate of the input features $\hbx$ for a new data point $\vbx$, 
the operator $P_f$ performs a mapping from the subspace $\mathcal{F}_W$ of input features the model can express (the row space of $W$) to the subspace  $\mathcal{F}_X$ of input features spanned by the training data set (the column space of $X$).
To analyze this mapping, we represent $P_f$ in terms of its singular value decomposition (SVD), 
\begin{equation}
P_f  =  \sum_{i: \sigma_i > 0} \sigma_i\vbf_{X, i}\vbf_{W, i}^T,\label{eq:SVD}
\end{equation}
where the sum ranges over the non-zero singular values $\sigma_i$.
The collections of  left and right singular vectors, $\vbf_{X, i}$ and $\vbf_{W, i}$, each form orthonormal bases within $\mathcal{F}_X$ and $\mathcal{F}_W$, respectively.
However, we note that it is possible for these two bases to span only a portion of their respective subspaces if they differ in dimension.
As a result, the vectors $\vbf_{X, i}$ span only a subspace of $\mathcal{F}_X$ if the model completely misses some of the input features spanned by the training data set.
Similarly, the vectors $\vbf_{W, i}$ span only a subspace of $\mathcal{F}_W$ if some of the input features the model can express are not actually used to learn anything from the training set.

To understand the behavior of $P_f$, 
we consider the mapping independently performed by each term in its SVD, depicted in Fig.~\ref{fig:feature_space}(d).
In the first step, $\vbx$ is mapped onto $\mathcal{F}_W$, resulting in the intermediate representation ${\hbx_W = \sum_i \hbx_{W, i} = \sum_i  \vbf_{W, i}\vbf_{W, i}^T\vbx}$.
Next, the data point is mapped from $\mathcal{F}_W$ to $\mathcal{F}_X$ resulting in ${\hbx = \sum_i \hbx_i = \sum_i \sigma_i\vbf_{X, i}\vbf_{W, i}^T \hbx_{W, i}= \sum_i \sigma_i\vbf_{X, i}\vbf_{W, i}^T\vbx}$.
Furthermore, we define the \textit{subspace orientation angle} $\theta_i$ between each pair of singular vectors, $\vbf_{X, i}$ and $\vbf_{W, i}$, shown in Fig.~\ref{fig:feature_space}(d).
Collectively, these angles characterize the relative orientation and overlap between the subspaces $\mathcal{F}_X$ and $\mathcal{F}_W$.
Later, we find that in special cases, data points are mapped orthogonally from $\mathcal{F}_W$ to $\mathcal{F}_X$.
To capture this behavior, we define the \textit{projection deviation angle} $\delta \phi_i$, which measures how this mapping deviates from  $90^\circ$ for each term in the SVD. As shown in Fig.~\ref{fig:feature_space}(d), we define $\delta \phi_i$ as $90^\circ$ minus the angle between the vectors $\vbf_{W, i}$ and $\hbx_i - \hbx_{W, i}$ (for $\vbx= \vbf_{W, i}$).

Next, we examine the properties of the geometric operator $P_f$ for specific choices of basis functions for the model features $\vbz(\vbx)$.

\paragraph{No Basis Functions: Orthogonal Projection}
The simplest set of model features in linear regression are the input features themselves (where $W = I_{N_f}$, the $N_f\times N_f$ identity matrix),
\begin{equation}
\vbz(\vbx) = \vbx\qqc P_f = X^+X=\sum_i \vbf_{X, i}\vbf_{X, i}^T.\label{eq:orthogonal}
\end{equation}
In this case, the operator $P_f$ reduces to an orthogonal projector in input feature space (${P_f^2 = P_f = P_f^T}$, with all $\theta_i = 0$ and $\delta\phi_i$ undefined, see SI for proof).

Because the model can express any input features ($\mathcal{F}_W = \mathcal{F}$),
the accuracy of the model's internal representation $\hbx$ is only limited by the input features sampled by the training data set via $X$.
In Fig.~\ref{fig:feature_space}(a), we show how the orthogonal projection takes place for a synthetic data set.
The model constructs $\hbx$ by orthogonally projecting $\vbx$ onto the subspace $\mathcal{F}_X$ of input features spanned by the training data, allowing the model to compare the new data point to data it encountered previously during training.

\paragraph{(Random) Linear Features: Oblique Projection}
Next, we consider a basis in which the model features are linearly related to the input features via a matrix $W$ [defined consistent with Eq.~\eqref{eq:zdecomp}],
\begin{equation}
\vbz(\vbx) = W^T\vbx\qqc P_f = (W(XW)^+X)^T =  \sum_i \frac{1}{\cos\theta_i}\vbf_{X, i}\vbf_{W, i}^T.\label{eq:oblique}
\end{equation}
For concreteness, we choose $W$ to be a random matrix so that the model is equivalent to a two-layer neural network with a linear activation function (with a random, fixed middle layer).
For this basis, $P_f$ becomes an oblique projector, satisfying the condition for a projector, ${P_f^2 = P_f}$, but not for orthogonality, ${P_f = P_f^T}$.
The singular vectors form mutually orthogonal bases such that ${\vbf_{W, i}\cdot\vbf_{X, j} = \delta_{ij}\cos\theta_i }$ with $\sigma_i = 1/\cos\theta_i$ and $\delta \phi_i =  0$ (see SI for proof).

Unlike the model features in the previous section, here the accuracy of the internal representation $\hbx$ for a data point $\vbx$ is limited by both the input features sampled by the training set via $X$ \textit{and} the input features that the model can express via $W$.
The resulting form of $P_f$ as an oblique projector reflects this behavior, depending on both the subspaces $\mathcal{F}_X$  and $\mathcal{F}_W$.
In Fig.~\ref{fig:feature_space}(b), we demonstrate how the oblique projection takes place in two steps for a synthetic data set (see Sec.~\ref{sec:dd} and SI). 
First, the input features $\vbx$ are orthogonally projected onto the subspace $\mathcal{F}_W$, 
resulting in an intermediate representation $\hbx_W$ in terms of input features the model can express.
Next, $\hbx_W$ is mapped onto the second subspace $\mathcal{F}_X$, but in a manner orthogonal to the \textit{first} subspace $\mathcal{F}_W$,
allowing the model to compare the current data point to those that it encountered previously in training.
As a result of this second step, information contained in $\vbx$ may be lost if the subspaces do not perfectly overlap. 
This occurs if any pair of left and right singular vectors have a subspace orientation angle $\theta_i > 0$.
In Sec.~\ref{sec:dd}, we discuss how this misalignment can cause large test errors at the interpolation threshold when one of the angles $\theta_i$ approaches $90^\circ$ and the model loses the ability to accurately express a direction in the training data.

\paragraph{(Random) Nonlinear Features: ``Noisy'' Oblique Projection}
Lastly, we consider nonlinear basis functions in which the model features take the form
\begin{equation}
\vbz(\vbx) = \varphi\qty(W^T\vbx)\qqc P_f = (WZ^+X)^T =  \sum_i \frac{\cos\delta\phi_i}{\cos(\theta_i + \delta\phi_i)}\vbf_{X, i}\vbf_{W, i}^T,\label{eq:noisy}
\end{equation}
where $\varphi$ is an arbitrary nonlinear activation function that is applied to each matrix element and $W$ is a matrix.
For example, ReLU activation corresponds to $\varphi(a) = C\max(0, a)$ where $C$ is a normalization constant defined so that Eq.~\eqref{eq:zdecomp} holds.
If we choose $W$ to be random, the model is equivalent to a two-layer neural network with a nonlinear activation function (with a random, fixed middle layer).
Unlike with the linear model features of the previous two sections, $P_f$ is no longer a true projector as the condition $P_f^2 = P_f$ generally no longer holds.
In addition, the left and right singular vectors are no longer mutually orthogonal and the singular values depend on both $\theta_i$ and $\delta\phi_i$ as shown above (see SI for proof).

Here, the accuracy of the internal representation $\hbx$ is again limited by both the input features spanned by the training data set and the input features that the model can express.
In Fig.~\ref{fig:feature_space}(c), we show how $P_f$ differs from the linear case for a synthetic data set with ReLU activation (see Sec.~\ref{sec:dd} and SI).
Similar to oblique projection, a new data point $\vbx$ is first orthogonally projected onto $\mathcal{F}_W$ to obtain the intermediate representation $\hbx_W$ in terms of input features the model can express.
However, the mapping from of $\hbx_W$ onto $\mathcal{F}_X$, used to compare the data point $\vbx$ to those encountered during training, is no longer orthogonal to $\mathcal{F}_W$. 
As a result, even if the model is capable of expressing all possible input features ($\mathcal{F}_W=\mathcal{F}$) and the input feature space is fully sampled during training  ($\mathcal{F}_X=\mathcal{F}$),
the internal representation $\hbx$ may not attain perfect accuracy.
Although $P_f$ is no longer formally an oblique projector, we observe that it does share some elements in common with one.
In Sec.~\ref{sec:dd}, we show that near the interpolation threshold, $P_f$ approximates to an oblique projector, resulting in large test errors at the interpolation threshold when one of the angles $\theta_i$ approaches $90^\circ$.
Furthermore, it is known that nonlinear model features in regression tend to behave like linear features with noise~\cite{Mei2019}.
Thus, we refer to $P_f$ as a ``noisy'' oblique projector in this case.

\section{Double-Descent in the Absence of Noise}\label{sec:dd}

\begin{figure*}[t]
\centering
\includegraphics[width=\linewidth]{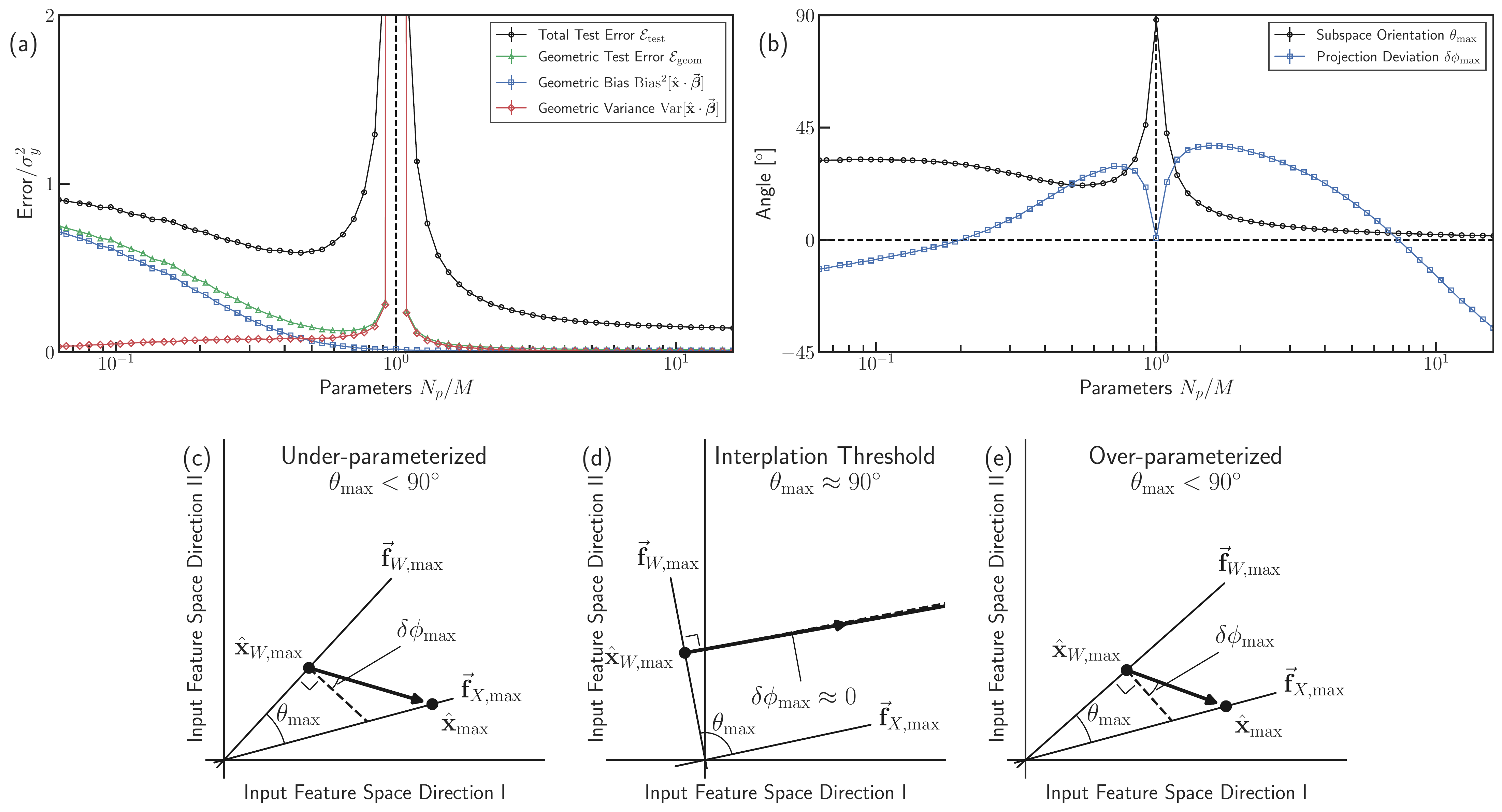}
\caption{\textbf{Double-descent phenomenon and corresponding geometric picture in input feature space.}
\textbf{(a)} The total test error (black circles), geometric test error (green triangles), geometric bias (blue squares), and geometric variance (red diamonds) for random nonlinear features as a function of the ratio of fit parameters $N_p$ to training data points $M$. The $y$-axis is scaled by the variance of the training labels $\sigma_y^2$.
\textbf{(b)} The subspace orientation angle $\theta_{\max}$ (black circles) and the projection deviation angle $\delta \phi_i$ (blue squares) associated with the maximum singular value in the SVD of $P_f$ for the model in (a). 
At the interpolation threshold (vertical dashed line), $\theta_{\max}$ approaches $90^\circ$ while $\delta \phi_i$ approaches $0^\circ$.
\textbf{(c)} In the under-parameterized regime ($N_p < M$),  $\theta_{\max} < 90^\circ$, resulting in low geometric variance.
\textbf{(d)} Near the interpolation threshold ($N_p \approx M$), $\theta_{\max} \approx 90^\circ$, resulting in large geometric variance.
\textbf{(e)} In the over-parameterized regime ($N_p > M$),  once again $\theta_{\max} < 90^\circ$ with low geometric variance. See Sec.~\ref{sec:dd} and SI for numerical details. }
\label{fig:data}
\vskip -1em
\end{figure*}

A surprising feature of double-descent in over-parameterized regression is that the phenomenon can occur even in the absence of label noise and nonlinearities in the labels or model features [i.e., $\delta y^*_{\mathrm{NL}}(\vbx)=0$ and $\varepsilon=0$ in Eq.~\eqref{eq:ydecomp} and $\delta \vbz_{\mathrm{NL}}(\vbx)=0$ in Eq.~\eqref{eq:zdecomp}]~\cite{Rocks2020}.
To show this, we use the decompositions of the labels in Eq.~\eqref{eq:ydecomp} and their predictions in Eq.~\eqref{eq:yhat_decomp} to define the \textit{geometric test error}  as $\mathcal{E}_{\mathrm{geom}} \equiv [\Delta \vbx \cdot \vbbeta]^2$, which depends on the vector quantity $\Delta \vbx  \equiv  ( I_{N_f} - P_f)\vbx$.
This quantity is analogous to the residual label error defined in Eq.~\eqref{eq:training_err}, 
so we interpret it as the residual error of the input feature estimate $\hbx$.
Furthermore, we decompose the geometric test error into \textit{geometric bias} and \textit{geometric variance} terms by averaging over the distribution of possible training sets $\mathcal{D}$ (see SI for proof),
\begin{equation}
\begin{gathered}
\E[\mathcal{D}]\qty[\mathcal{E}_{\mathrm{geom}}] =  \Bias^2[\hbx\cdot\vbbeta] + \Var[][\hbx\cdot\vbbeta],\\
\Bias[\hbx\cdot\vbbeta] \equiv \E[\mathcal{D}][\Delta \vbx\cdot\vbbeta] \qqc
\Var[][\hbx\cdot\vbbeta] \equiv \E[\mathcal{D}][(\hbx\cdot\vbbeta)^2] - \E[\mathcal{D}][\hbx\cdot\vbbeta]^2.\label{eq:bv_decomp}
\end{gathered}
\end{equation}
In the absence of label noise and nonlinearities, the traditional definitions~\cite{Bishop2006} of the test error, bias, and variance reduce to these geometric versions (see SI for proof).

In Fig.~\ref{fig:data}(a), we plot these three quantities, along with the total test error ${\mathcal{E}_{\mathrm{test}} = [y(\vbx) - \hat{y}(\vbx)]^2}$,
for a basis of random nonlinear features with ReLU activation $\smash{\varphi(x) = 2\max(0, x)}$ for a synthetic data set with linear labels $\smash{y^*(\vbx) = \vbx\cdot \vbbeta}$.
We sample the input features $\vbx$ and label noise $\varepsilon$ for the training and test data sets (each of size $M=512$) i.i.d. from normal distributions with mean zero and variances $\sigma_X^2/N_f$ and $\sigma_\varepsilon^2$, respectively.
For each simulation, we also sample the parameters $\smash{\vbbeta}$ and matrix $W$ i.i.d. from normal distributions with mean zero and variances $\smash{\sigma_W^2/N_p}$ and $\sigma_\beta^2$, respectively.
We choose a signal-to-noise ratio for the labels of $\smash{\sigma_X^2\sigma_\beta^2/\sigma_\varepsilon^2 = 10}$ and a regularization parameter of $\lambda=10^{-8}$. We use this general setup to construct all numerical examples and schematics presented in this work. All points in error curves are averaged over at least 500 independent simulations with small error bars indicating the error on the mean.
We also include more comprehensive numerical results for all three models in the SI.

As expected, we find that the divergence in the geometric test error stems from the geometric variance at the interpolation threshold.
Furthermore, we find that the over-parameterized regime performs better than its under-parameterized counterpart due to a monotonic decrease in the geometric bias as the number of fit parameters increases.
From Eq.~\eqref{eq:bv_decomp}, we see that the bias arises from residual input feature error. 
Just as  $\Delta y$ decreases as $P_\ell$ approaches the identity $I_M$, resulting in lower training error,
$\Delta \vbx$ decreases when $P_f$ approaches the identity $I_{N_f}$, resulting in lower bias (see SI for numerics).

The overfitting that gives rise to the double-descent behavior occurs when a model is not able to accurately express directions in the training data using its model features. 
To see this in our geometric picture, we make use of the observation that the test error diverges at the interpolation threshold when $Z$ contains a small singular value,
or alternatively, $P_f$ contains a large singular value~\cite{Advani2020}. 
Therefore, we focus on the maximum singular value $\sigma_{\max}$ which diverges and dominates the SVD of $P_f$ in Eq.~\eqref{eq:SVD} near the interpolation threshold (see SI for numerics).
In Fig~\ref{fig:data}(b), we plot the subspace orientation angle $\theta_{\max}$ and projection deviation angle $\delta \phi_{\max}$ associated with $\sigma_{\max}$.
We find that near the interpolation threshold, $\theta_{\max}$ approaches $90^\circ$, while $\delta \phi_{\max}$ approaches $0^\circ$,
and as a result, $P_f$ approaches a true oblique projector (see SI for proof).

\begin{wrapfigure}{R}{0.5\linewidth}
\vskip -1.0em
\centering
\includegraphics[width=\linewidth]{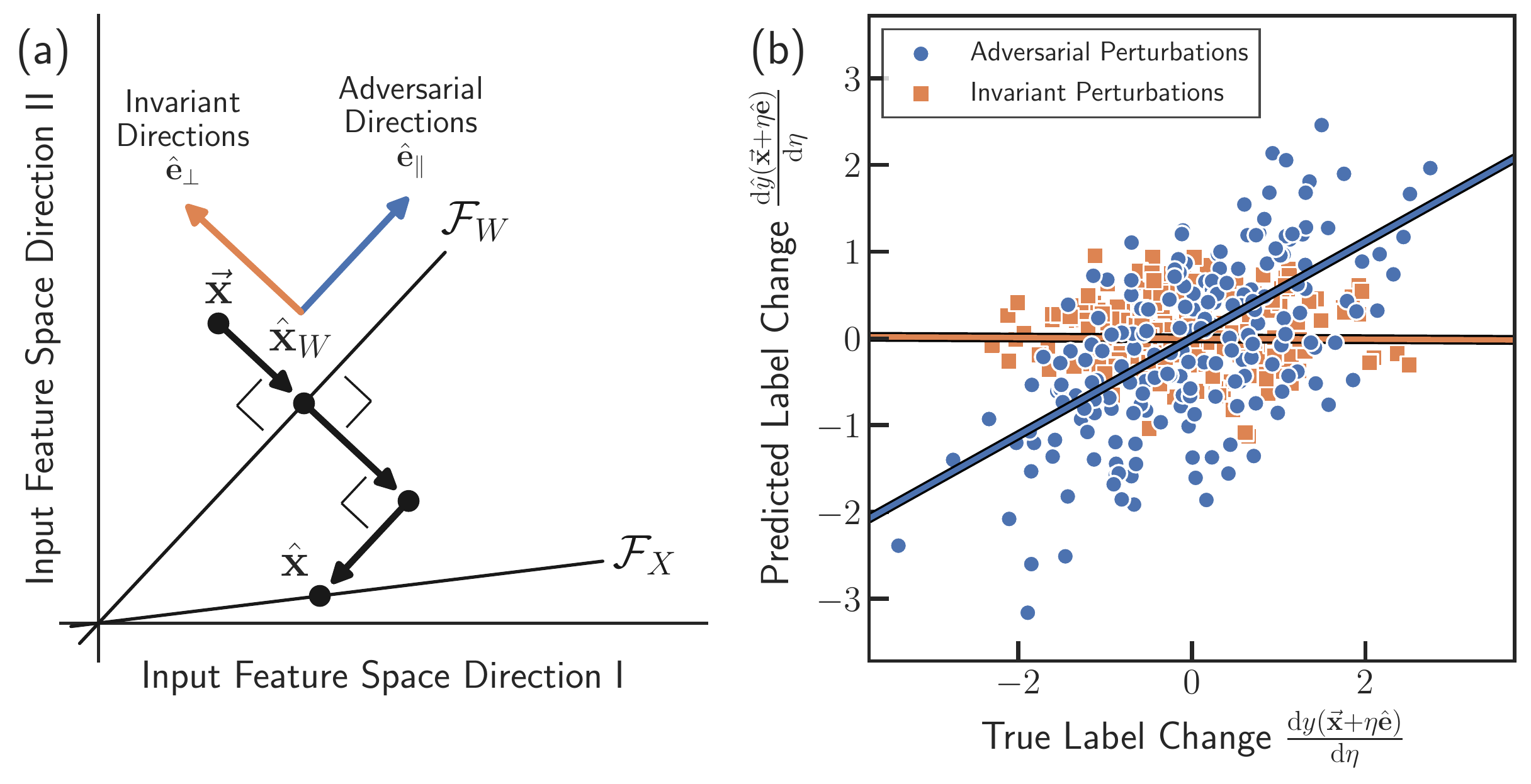}
\caption{{\bf Adversarial perturbations}.
{\bf (a)} A perturbation to the input features of a data point can be decomposed into  ``adversarial'' and
``invariant'' directions, parallel and perpendicular, respectively, to the subspace of input features that can be captured by the model. 
{\bf (b)} The change in the predicted and true labels for $200$ random adversarial (blue) and $200$ random invariant (orange) perturbations applied to the same data point with correlations shown as lines (see Sec.~\ref{sec:dd} and SI for numerical details).}
\label{fig:adversarial_perturb}
\vskip -4.0em
\end{wrapfigure}

Most importantly, we find that test error arises when directions in the subspaces $\mathcal{F}_W$ and $\mathcal{F}_X$ related via $P_f$ become misaligned $\theta_{\max} > 0$.
In Figs.~\ref{fig:data}(c)-(d), we illustrate this behavior for the singular vectors and angles associated with $\sigma_{\max}$.
When corresponding directions in the subspaces $\mathcal{F}_W$ and $\mathcal{F}_X$ do not perfectly overlap, 
the model is not able to accurately express the training data.
The closer $\theta_{\max}$ is to $90^\circ$, the worse the model becomes as it tends to overfit the poorly expressed directions in the training data,
resulting in large geometric variance and test errors.

\section{Adversarial Perturbations}

Our geometric picture also yields insights into the nature of adversarial perturbations. 
Inspired by Ref.~\cite{Goodfellow2015},
we consider perturbing a data point, $\vbx \rightarrow \vbx + \eta \hat{\mathbf{e}}$, where $\hat{\mathbf{e}}$ is a random unit vector and $\eta$ measures the magnitude of the perturbation. 
As depicted in Fig.~\ref{fig:adversarial_perturb}(a), any perturbation can be decomposed into components parallel and perpendicular to the kernel of $P_f$. 
Mathematically, this is the statement that we can write ${\hat{\mathbf{e}}= \hat{\mathbf{e}}_\parallel+ \hat{\mathbf{e}}_\perp}$ where  ${\sum_i \vbf_{W, i}\vbf_{W, i}^T\hat{\mathbf{e}}_\parallel = \hat{\mathbf{e}}_\parallel}$ and ${\sum_i \vbf_{W, i}\vbf_{W, i}^T\hat{\mathbf{e}}_\perp = 0}$. 
Our picture suggests that perturbations $\hat{\mathbf{e}} = \hat{\mathbf{e}}_\parallel$,  should have a large effect on the label predictions $\hat{y}$, whereas the model should be insensitive to perturbations $\hat{\mathbf{e}} = \hat{\mathbf{e}}_\perp$, which result in similar internal representations $\hbx$ (with any variations coming entirely from noise and nonlinear effects). 
For this reason, the complement to the kernel of $P_f$ is the subspace of adversarial directions and the kernel of $P_f$ is that of invariant directions.
Furthermore, Eq.~\eqref{eq:bv_decomp} indicates that the latter subspace is a result of nonzero geometric bias.

To test this intuition, we calculate how the model predictions $\hat{y}$ and true labels $y$ change in response to small random perturbation from both subspaces.
In Fig.~\ref{fig:adversarial_perturb}(b), we plot the derivatives $\dv{\hat{y}(\vbx+ \eta \hat{\mathbf{e}})}{\eta}$ and $\dv{y(\vbx+ \eta \hat{\mathbf{e}})}{\eta}$ 
for random perturbations that are purely adversarial or invariant (see SI for details).
We find that adversarial perturbations tend to significantly change the predicted labels when the true labels change (correlation is shown as a blue line), 
whereas invariant perturbations on average do not result in changes to the predicted labels that correlate with the changes in true labels (orange line).

\section{Conclusions}

In this paper, we presented an alternative geometric interpretation of regression, residing in the space of input features so that it applies both under- and over-parameterized models.
This new picture provides a geometric interpretation of the double-descent phenomenon and suggests a natural scheme for decomposing perturbations into adversarial and invariant directions. 
It will be interesting to see which, if any, of these geometric intuitions can be extended to more complicated settings such as deep neural networks and classification tasks.

\section*{Broader Impacts}
The authors do not believe this theoretical work will raise any ethical concerns or will generate any adverse future societal consequences.

\section*{Acknowledgments}
This work was supported by NIH NIGMS grant 1R35GM119461 and a Simons Investigator in the Mathematical Modeling of Living Systems (MMLS) award to PM. 
The authors also acknowledge support from the Shared Computing Cluster administered by Boston University Research Computing Services.

%

\clearpage

\section*{\centering Supporting Information: \\ The Geometry of Over-parameterized Regression and Adversarial Perturbations}

\setcounter{section}{0}
\renewcommand{\thesection}{S\arabic{section}}
\renewcommand{\thetable}{S\arabic{table}}
\renewcommand{\thefigure}{S\arabic{figure}}

\section{Mathematical Proofs}

\numberwithin{equation}{section}
\newtheorem{theorem}{Theorem}

\begin{theorem}
In the label decomposition of Eq.~\eqref{eq:ydecomp}, the linear term $\vbx\cdot\vbbeta$ and the nonlinear term $\delta y^*_{\mathrm{NL}}(\vbx)$ are statistically independent with respect to the distribution of input features $\vbx$.
\end{theorem}

\begin{proof}
We show the covariance of $\vbx\cdot\vbbeta$ and $\delta y^*_{\mathrm{NL}}(\vbx)$ is zero:
\begin{equation}
\begin{aligned}
\Cov[\vbx][\vbx\cdot\vbbeta, \delta y^*_{\mathrm{NL}}(\vbx)] &= \Cov[\vbx][\vbx\cdot\vbbeta, y^*(\vbx) - \vbx\cdot\vbbeta]\\
&= \vbbeta\cdot\Cov[\vbx][\vbx, y^*(\vbx) ] - \vbbeta\cdot\Cov[\vbx][\vbx, \vbx^T ]\vbbeta\\
&= \vbbeta\cdot\Sigma_{\vbx}\vbbeta - \vbbeta\cdot\Sigma_{\vbx}\vbbeta\\
&= 0.
\end{aligned}
\end{equation}
\end{proof}

\begin{theorem}
In the model feature decomposition of Eq.~\eqref{eq:zdecomp}, the linear term $W^T\vbx$ and the nonlinear term $\delta\vbz_{\mathrm{NL}}(\vbx)$ are statistically independent with respect to the distribution of input features $\vbx$.
\end{theorem}

\begin{proof}
We show the covariance of $W^T\vbx$ and $\delta\vbz_{\mathrm{NL}}(\vbx)$ is zero:
\begin{equation}
\begin{aligned}
\Cov[\vbx][W^T\vbx, \delta\vbz_{\mathrm{NL}}(\vbx)^T] &= \Cov[\vbx][W^T\vbx, \vbz^T(\vbx) - \vbx^TW]\\
&= W^T\Cov[\vbx][\vbx, \vbz^T(\vbx) ] - W^T\Cov[\vbx][\vbx, \vbx^T ]W\\
&= W^T\Sigma_{\vbx}W - W^T\Sigma_{\vbx}W\\
&= 0.
\end{aligned}
\end{equation}
\end{proof}

\begin{theorem}
The label predictions can be expressed in terms of the geometric decomposition in Eq.~\eqref{eq:yhat_decomp}. 
\end{theorem}

\begin{proof}
We start with the OLS solution for the label predictions in Eq.~\eqref{eq:solution} and then substitute the label decomposition in Eq.~\eqref{eq:ydecomp} and model feature decomposition in Eq.~\eqref{eq:zdecomp} to get
\begin{equation}
\begin{aligned}
\hat{y}(\vbx) &= \vbz(\vbx)\cdot Z^+\vby\\
&= (W^T\vbx + \delta \vbz_{\mathrm{NL}}(\vbx))\cdot Z^+ (X\vbbeta + \delta \vby^*_{\mathrm{NL}} + \vbeps)\\
&= \vbx\cdot WZ^+X \vbbeta +  \delta \vbz_{\mathrm{NL}}(\vbx)^TZ^+(X\vbbeta + \delta \vby^*_{\mathrm{NL}} + \vbeps) + \vbx^T WZ^+(\delta \vby^*_{\mathrm{NL}} + \vbeps)\\
&= \vbx\cdot P_f^T \vbbeta +  \delta \vbz_{\mathrm{NL}}(\vbx)^TZ^+\vby + \vbx^T WZ^+(\delta \vby^*_{\mathrm{NL}} + \vbeps)\\
&= \hbx\cdot \vbbeta + \delta \hat{y}(\vbx),
\end{aligned}
\end{equation}
where where have defined the vector $\hbx \equiv P_f\vbx$, the matrix operator $P_f \equiv ( WZ^+X)^T$, and the scalar quantity ${\delta \hat{y}(\vbx) \equiv  \delta \vbz_{\mathrm{NL}}(\vbx)^TZ^+\vby + \vbx^T WZ^+(\delta \vby^*_{\mathrm{NL}} + \vbeps)}$.
\end{proof}

\begin{theorem}
The operator $P_f=X^+X$ for model features without basis functions is an orthogonal projector whose singular value decomposition is
\begin{equation}
P_f = \sum_i \vbf_{X, i}\vbf_{X, i}^T
\end{equation}
with subspace orientation angles $\theta_i = 0$ and undefined projection deviation angles $\delta\phi_i$.
\end{theorem}

\begin{proof}
We start with the general form of the singular value decomposition of $P_f$, given in Eq.~\eqref{eq:SVD},
\begin{equation}
P_f = \sum_i \sigma_i \vbf_{X, i}\vbf_{W, i}^T,
\end{equation}
where $\vbf_{W, i}\cdot \vbf_{W, j} = \delta_{ij}$, $\vbf_{X, i}\cdot \vbf_{X, j} = \delta_{ij}$ and $\sigma_i > 0$.

It is evident that $P_f$ is an orthogonal projector since 
\begin{equation}
P_f^2 = X^+XX^+X = X^+X = P_f
\end{equation}
and 
\begin{equation}
P_f^T = (X^+X)^T = X^+X = P_f
\end{equation}
using the properties of the pseudoinverse.

Combining these two properties,
\begin{equation}
\begin{aligned}
 P_f &=  P_f^TP_f\\
\sum_j \sigma_j \vbf_{X, j}\vbf_{W, j}^T &= \sum_{jk} \sigma_j \sigma_k \vbf_{W, j}\vbf_{X, j}^T\vbf_{X, k}\vbf_{W, k}^T\\
\sum_j \sigma_j \vbf_{X, j}\vbf_{W, j}^T &= \sum_k \sigma_k^2\vbf_{W, k}\vbf_{W, k}^T.
\end{aligned}
\end{equation}
We then right-multiply this equation by the vector $\vbf_{W, i}$ to find
\begin{equation}
\begin{aligned}
\sigma_i\vbf_{X, i} &= \sigma_i^2 \vbf_{W, i}\\
\vbf_{X, i} &= \sigma_i \vbf_{W, i}.
\end{aligned}
\end{equation}
Since these two vector are unit vectors and $\sigma_i$ is positive, it must be that $\vbf_{X, i}= \vbf_{W, i}$ and $\sigma_i = 1$,
resulting in the singular value decomposition
\begin{equation}
P_f = \sum_i \vbf_{X, i}\vbf_{X, i}^T.
\end{equation}

The subspace orientation angles $\theta_i$ are found using their definition,
\begin{equation}
 \cos\theta_i = \vbf_{X, i}\cdot \vbf_{W, i} = 1,
\end{equation}
indicating that $\theta_i = 0$.

In addition, the projection deviation angles $\delta \phi_i$ are defined as
\begin{equation}
\cos\qty(\frac{\pi}{2} - \delta\phi_i) = \frac{\vbf_{W, i}\cdot(\hbx_i - \hbx_{W, i})}{\norm{\hbx_i - \hbx_{W, i}}}
\end{equation}
with $\vbx = \vbf_{W, i}$.
However, the vector in the numerator has zero magnitude,
\begin{equation}
\hbx_i - \hbx_{W, i} = \sigma_i\vbf_{X, i} - \vbf_{W, i} = 0,
\end{equation}
so the angles $\delta \phi_i$ are undefined.
\end{proof}

\begin{theorem}
The operator $P_f=(W(XW)^+X)^T$ for linear features is an oblique projector whose singular value decomposition is
\begin{equation}
P_f = \sum_i \frac{1}{\cos\theta_i}\vbf_{X, i}\vbf_{W, i}^T
\end{equation}
with $\vbf_{W, i}\cdot\vbf_{X, i}=\delta_{ij} \cos\theta_i $ and   $\delta\phi_i = 0$.
\end{theorem}

\begin{proof}
We start with the general form of the singular value decomposition of $P_f$, given in Eq.~\eqref{eq:SVD},
\begin{equation}
P_f = \sum_i \sigma_i \vbf_{X, i}\vbf_{W, i}^T,
\end{equation}
where $\vbf_{W, i}\cdot \vbf_{W, j} = \delta_{ij}$, $\vbf_{X, i}\cdot \vbf_{X, j} = \delta_{ij}$ and $\sigma_i > 0$.

Clearly, $P_f$ is an oblique projector since
\begin{equation}
\begin{aligned}
P_f^2 &= (W(XW)^+X)^T(W(XW)^+X)^T\\
& = X^T(W^TX^T)^+W^TX^T(W^TX^T)^+W^T\\
& = X^T(W^TX^T)^+W^T\\
& = (W(XW)^+X)^T\\
& = P_f.
\end{aligned}
\end{equation}

Using this property, we find
\begin{equation}
\begin{aligned}
 P_f &= P_f^2\\
\sum_k \sigma_k \vbf_{X, k}\vbf_{W, k}^T &= \sum_{kl} \sigma_k \sigma_l \vbf_{X, k}\vbf_{W, k}^T\vbf_{X, l}\vbf_{W, l}^T.
\end{aligned}
\end{equation}
We then left-multiply by $\vbf_{X, i}^T$ and right-multiply by $\vbf_{W, j}$, giving us
\begin{equation}
\begin{aligned}
\sigma_i\delta_{ij} &= \sigma_i^2\vbf_{W, i}^T\vbf_{X, j},
\end{aligned}
\end{equation}
from which we find that $\vbf_{W, i}\cdot\vbf_{X, j} = \frac{1}{\sigma_i}\delta_{ij}$, where we have used the fact that $\sigma_i$ is positive to choose the sign.

Based on the definition of the subspace orientation angles $\theta_i$, 
\begin{equation}
\cos\theta_i = \vbf_{X, i}\cdot \vbf_{W, i}  = \frac{1}{\sigma_i},
\end{equation}
so the singular value decomposition is
\begin{equation}
P_f = \sum_i \frac{1}{\cos\theta_i}\vbf_{X, i}\vbf_{W, i}^T
\end{equation}
with $\vbf_{W, i}\cdot\vbf_{X, j} = \cos\theta_i \delta_{ij}$.

We split the analysis of the projection deviation angles $\delta \phi_i$ into two cases.
First, when $\theta_i > 0$, we have
\begin{equation}
\begin{aligned}
\cos\qty(\frac{\pi}{2} - \delta\phi_i) &= \frac{\vbf_{W, i}\cdot(\hbx_i - \hbx_{W, i})}{\norm{\hbx_i - \hbx_{W, i}}}\\
&= \frac{\vbf_{W, i}\cdot(\sigma_i\vbf_{X, i} - \vbf_{W, i})}{\norm*{\sigma_i\vbf_{X, i} - \vbf_{W, i}}}\\
&= \frac{1- 1}{\norm*{\sigma_i\vbf_{X, i} - \vbf_{W, i}}}\\
&= 0.
\end{aligned}
\end{equation}
\end{proof}
However, when $\theta_i = 0$, the denominator $\norm{\hbx_i - \hbx_{W, i}}$ is zero, so $\delta \phi_i$ is undefined. In this case, we are free to take $\delta \phi_i$ to be zero.

\begin{theorem}
The operator $P_f = (WZ^+X)^T$ for arbitrary nonlinear features has a singular value decomposition of the form
\begin{equation}
P_f = \sum_i \frac{\cos\delta\phi}{\cos(\theta_i+\delta\phi_i)}\vbf_{X, i}\vbf_{W, i}^T.
\end{equation}
\end{theorem}

\begin{proof}

We start with the definition of the projection deviation angles $\delta \phi_i$,
\begin{equation}
\begin{aligned}
\cos\qty(\frac{\pi}{2} - \delta\phi_i) &= \frac{\vbf_{W, i}\cdot(\hbx_i - \hbx_{W, i})}{\norm{\hbx_i - \hbx_{W, i}}}\\
&= \frac{\vbf_{W, i}\cdot(\sigma_i\vbf_{X, i}-\vbf_{W, i})}{\norm*{\sigma_i\vbf_{X, i}-\vbf_{W, i}}}\\
&=  \frac{\sigma_i\cos\theta_i-1}{\sqrt{1+\sigma_i^2-2\sigma_i\cos\theta_i}}
\end{aligned}
\end{equation}
where we have also used the definition of the subspace orientation angle $\theta_i$ as $\cos\theta_i = \vbf_{X, i}\cdot \vbf_{W, i}$.
Next, we solve the above equation for $\sigma_i$ to find
\begin{equation}
\begin{aligned}
\sigma_i &=  \frac{\sin\qty(\frac{\pi}{2} - \delta\phi_i)}{\sin\qty(\theta_i +\frac{\pi}{2} - \delta\phi_i)}\\
&= \frac{\cos \delta\phi_i}{\cos(\theta_i + \delta\phi_i)},
\end{aligned}
\end{equation}
where we have chosen the sign so that singular values diverge when $\theta_i+\delta\phi_i=\frac{\pi}{2}$.

The singular value decomposition is then 
\begin{equation}
P_f = \sum_i \frac{\cos\delta\phi}{\cos(\theta_i+\delta\phi_i)}\vbf_{X, i}\vbf_{W, i}^T.
\end{equation}
\end{proof}

\begin{theorem}
The geometric test error can be decomposed into the geometric bias and geometric variance as defined in  Eq.~\eqref{eq:bv_decomp}.
\end{theorem}

\begin{proof}
To derive the geometric bias-variance decomposition in Eq.~\eqref{eq:bv_decomp}, we start with the definition of the geometric test error. We find
\begin{equation}
\begin{aligned}
\E[\mathcal{D}][\mathcal{E}_{\mathrm{geom}}] &= \E[\mathcal{D}][(\Delta \vbx\cdot \vbbeta)^2]\\
&= \E[\mathcal{D}][(\hbx \cdot \vbbeta - \vbx\cdot \vbbeta)^2]\\
&= \E[\mathcal{D}][(\hbx \cdot \vbbeta - \E[\mathcal{D}][\hbx \cdot \vbbeta] +  \E[\mathcal{D}][\hbx \cdot \vbbeta] - \vbx\cdot \vbbeta)^2]\\
&= (\E[\mathcal{D}][\hbx \cdot \vbbeta] - \vbx\cdot \vbbeta)^2 +\E[\mathcal{D}][(\hbx \cdot \vbbeta - \E[\mathcal{D}][\hbx \cdot \vbbeta])^2]\\
&\qquad + 2(\E[\mathcal{D}][\hbx \cdot \vbbeta] - \vbx\cdot \vbbeta)\E[\mathcal{D}][\hbx \cdot \vbbeta - \E[\mathcal{D}][\hbx \cdot \vbbeta]]\\
&= (\E[\mathcal{D}][\hbx \cdot \vbbeta] - \vbx\cdot \vbbeta)^2 +\E[\mathcal{D}][(\hbx \cdot \vbbeta)^2] - \E[\mathcal{D}][\hbx \cdot \vbbeta]^2\\
&= \Bias^2[\hbx \cdot \vbbeta]+ \Var[][\hbx \cdot \vbbeta],
\end{aligned}
\end{equation}
where the geometric bias and variance are defined as 
\begin{equation}
\begin{aligned}
\Bias[\hbx \cdot \vbbeta] &= \E[\mathcal{D}][\hbx \cdot \vbbeta] - \vbx\cdot \vbbeta,\\
\Var[][\hbx \cdot \vbbeta] &= \E[\mathcal{D}][(\hbx \cdot \vbbeta)^2] - \E[\mathcal{D}][\hbx \cdot \vbbeta]^2.
\end{aligned}
\end{equation}
\end{proof}

\begin{theorem}
The geometric test error, geometric bias and geometric variance reduce to the standard definitions as presented in Ref.~\cite{Bishop2006} in the absence of label noise and nonlinearities.
\end{theorem}
\begin{proof}
We substitute the decompositions for the true and predicted labels in Eqs.~\eqref{eq:ydecomp} and \eqref{eq:yhat_decomp} into the standard definitions of test error, bias and variance, allowing us  to write them in terms of their geometric counterparts. The total test error averaged over the training data set $\mathcal{D}$ and the label noise of the test data point $\varepsilon$ (assumed to have zero mean) is
\begin{equation}
\begin{aligned}
\E[\mathcal{D}, \varepsilon][\mathcal{E}_{\mathrm{test}}] &= \E[\mathcal{D}, \varepsilon][(y(\vbx) -\hat{y}(\vbx))^2]\\
 &= \E[\mathcal{D}, \varepsilon][(\vbx\cdot\vbbeta + \delta y^*_{\mathrm{NL}}(\vbx) + \varepsilon - \hbx\cdot\vbbeta - \delta\hat{y}(\vbx))^2]\\
&= \E[\mathcal{D}, \varepsilon][(\Delta\vbx\cdot\vbbeta + \delta y^*_{\mathrm{NL}}(\vbx) + \varepsilon - \delta\hat{y}(\vbx))^2]\\
&= \E[\mathcal{D}][(\Delta\vbx\cdot\vbbeta)^2] + \E[\mathcal{D}][(\delta y^*_{\mathrm{NL}}(\vbx)- \delta\hat{y}(\vbx))^2]\\
&\qquad + \E[\varepsilon][\varepsilon^2] + 2\E[\mathcal{D}][(\Delta\vbx\cdot\vbbeta)(\delta y^*_{\mathrm{NL}}(\vbx)- \delta\hat{y}(\vbx))]\\
&= \E[\mathcal{D}][\mathcal{E}_{\mathrm{geom}}] + \E[\mathcal{D}][(\delta y^*_{\mathrm{NL}}(\vbx)- \delta\hat{y}(\vbx))^2]\\
&\qquad + \E[\varepsilon][\varepsilon^2] + 2\E[\mathcal{D}][(\Delta\vbx\cdot\vbbeta)(\delta y^*_{\mathrm{NL}}(\vbx)- \delta\hat{y}(\vbx))].
\end{aligned}
\end{equation}
The bias is
\begin{equation}
\begin{aligned}
\Bias[\hat{y}(\vbx)] &= \E[\mathcal{D}][\hat{y}(\vbx)] - y^*(\vbx)\\
&=  \E[\mathcal{D}][\hbx\cdot\vbbeta + \delta\hat{y}(\vbx)] - \vbx\cdot\vbbeta - \delta y^*_{\mathrm{NL}}(\vbx)\\
&= -\E[\mathcal{D}][\Delta \vbx\cdot\vbbeta] +  \E[\mathcal{D}][\delta\hat{y}(\vbx)]- \delta y^*_{\mathrm{NL}}(\vbx)\\
&= -\Bias[\hbx\cdot\vbbeta ] +  \E[\mathcal{D}][\delta\hat{y}(\vbx)]- \delta y^*_{\mathrm{NL}}(\vbx).
\end{aligned}
\end{equation}
The variance is
\begin{equation}
\begin{aligned}
\Var[][\hat{y}(\vbx)] &= \Var[\mathcal{D}][\hat{y}(\vbx)]\\
&=  \Var[\mathcal{D}][\hbx\cdot\vbbeta + \delta\hat{y}(\vbx)]\\
&= \Var[\mathcal{D}][\hbx\cdot\vbbeta] +  \Var[\mathcal{D}][\delta\hat{y}(\vbx)] + 2\Cov[\mathcal{D}][\hbx\cdot\vbbeta, \delta\hat{y}(\vbx)].
\end{aligned}
\end{equation}
Clearly, these three quantities reduce to their geometric counterparts when $\varepsilon$, $\delta y^*_{\mathrm{NL}}(\vbx)$ and $\delta\vbz_{\mathrm{NL}}(\vbx)$ are zero.
\end{proof}

\begin{theorem}
When the maximum single singular value $\sigma_{\max}$ of $P_f$ diverges, the subspace orientation angle $\theta_{\max}$ approaches $90^\circ$, and the projection deviation angle $\delta \phi_{\max}$ approaches $0^\circ$, $P_f$ is approximately an oblique projector for arbitrary nonlinear model features.
\end{theorem}
\begin{proof}
We start with the general form of the singular value decomposition of $P_f$, given in Eq.~\eqref{eq:SVD},
\begin{equation}
P_f = \sum_i \sigma_i \vbf_{X, i}\vbf_{W, i}^T,
\end{equation}
where $\vbf_{W, i}\cdot \vbf_{W, j} = \delta_{ij}$, $\vbf_{X, i}\cdot \vbf_{X, j} = \delta_{ij}$ and $\sigma_i > 0$.

Now suppose a single singular value becomes much larger than the rest, $\sigma_{\max}$. In this case, the singular value decomposition approximates to
\begin{equation}
P_f \approx \sigma_{\max} \vbf_{X, \max}\vbf_{W, \max}^T.
\end{equation}
In addition if $\delta \phi_{\max}\approx 0^\circ$, then $\sigma_{\max}\approx \frac{1}{\cos\theta_{\max}}$, so
\begin{equation}
P_f \approx \frac{1}{\cos\theta_{\max}}\vbf_{X, \max}\vbf_{W, \max}^T.
\end{equation}

Squaring this operator then results in
\begin{equation}
\begin{aligned}
P_f^2 &= \frac{1}{\cos^2\theta_{\max}}\vbf_{X, \max}\vbf_{W, \max}^T\vbf_{X, \max}\vbf_{W, \max}^T\\
&= \frac{1}{\cos\theta_{\max}}\vbf_{X, \max}\vbf_{W, \max}^T\\
&= P_f,
\end{aligned}
\end{equation}
so $P_f$ is approximately an oblique projector.
\end{proof}

\section{Numerical Details}

\subsection{Subspace Schematics}

To generate the schematics of the training label space in Figs.~\ref{fig:data_space}(b)-(c) and the input feature space in Figs.~\ref{fig:feature_space}(a)-(c) and \ref{fig:adversarial_perturb}(a), we used the numerical scheme detailed in Sec.~\ref{sec:dd}.
In each of these figures, we plot the relative magnitudes of the  $\vbx$, $\hbx_W$, and $\hbx$ and the angles between them to scale.
We then used these vectors to orient the lines representing the different subspaces.
In each figure, we only depict the portions of the subspaces that have a one-to-one mapping via $P_f$, since these are the only directions that contribute to $\hbx$.

We note that in Figs.~\ref{fig:feature_space}(b)-(c) and \ref{fig:adversarial_perturb}(a) the vector $\vbx$, $\hbx_W$, and $\hbx$ are drawn to lie in the same plane for visual clarity, even though this is generally not the case.
While the angles between adjacent vectors are depicted to scale, the angles between non-adjacent vectors are not drawn accurately.

In addition, the intermediate vector in Figs.~\ref{fig:feature_space}(c) and \ref{fig:adversarial_perturb}(a), which lies between the subspaces $\mathcal{F}_X$ and $\mathcal{F}_W$, generally does not lie in the same place as $\hbx_W$ and $\hbx$.
In this case, the vector is an approximate 2$d$ representation found by finding the component of $\hbx-\hbx_W$ that is perpendicular to $\hbx_W$ and lies in the same plane as both  $\hbx_W$ and $\hbx$.

In the schematics of the singular vectors in Figs.~\ref{fig:feature_space}(c) and \ref{fig:data}(c)-(d), the relative magnitudes of the vectors $\hbx_{W, i}$ and $\hbx_i$ and the angles $\theta_i$ and $\delta\phi_i$ are drawn to scale.
In each figure, we show the singular vectors and angles corresponding to the largest singular value of $P_f$.

In Table~\ref{tab:schematics}, we report the model, the ratio of input features $N_p$ to training data points, and the ratio of fit parameters $N_p$ to training data points used to generate each schematic. In each case we use $M=512$ training data points and round $N_f$ and $N_p$ down to the nearest integer.

\begin{table}[h]
  \caption{{\bf Model Details for Projector Schematics}}
  \label{tab:schematics}
  \centering
  \begin{tabular}{llccc}
    \toprule
    Figure     & Model     & Input Features $N_f/M$ & Fit Parameters $N_p/M$ \\
    \midrule
    \ref{fig:data_space}(b) & Random Nonlinear Features  & $1/4$ &  $1/4$\\
    \ref{fig:data_space}(c) & Random Nonlinear Features  &  $1/4$ &  $3$ \\
    \ref{fig:feature_space}(a) & No Basis Functions    & $6/5$ & $6/5$ \\
    \ref{fig:feature_space}(b) & Random Linear Features   & $6/5$ & $2$ \\
     \ref{fig:feature_space}(c) & Random Nonlinear Features   & $6/5$ & $2$ \\
     \ref{fig:feature_space}(d) & Random Nonlinear Features   & $1/4$ & $6/5$ \\
     \ref{fig:data}(c) & Random Nonlinear Features  & $1/4$ & $5/6$ \\
    \ref{fig:data}(d) & Random Nonlinear Features  & $1/4$ & $1$ \\
    \ref{fig:data}(e) & Random Nonlinear Features  & $1/4$ & $6/5$ \\
    \ref{fig:adversarial_perturb}(a) & Random Nonlinear Features  & $6/5$ & $3$ \\
    \bottomrule
  \end{tabular}
\end{table}

\subsection{Bias-Variance Decompositions}

Here, we describe our scheme to efficiently generate plots of the geometric bias and geometric variance.
First, we express the squared geometric bias and geometric variance for a single test data point in terms of averages over two independent training data sets as 
\begin{equation}
\begin{aligned}
\Bias^2[\hbx\cdot\vbbeta] &= \E[\mathcal{D}_1, \mathcal{D}_2][(\Delta \vbx_1\cdot\vbbeta)(\Delta \vbx_2\cdot\vbbeta)]\\
\Var[][\hbx\cdot\vbbeta] &= \E[\mathcal{D}_1][(\hbx\cdot\vbbeta)^2] -  \E[\mathcal{D}_1, \mathcal{D}_2][( \vbx_1\cdot\vbbeta)(\vbx_2\cdot\vbbeta)]\label{eq:bv_approx}
\end{aligned}
\end{equation}
where $\hbx_1$ and $\hbx_2$ are the label predictions from training on the different training data sets.

As suggested by these formulas, during each simulation, we independently generate two training data sets of equal size, $\mathcal{D}_1$ and $\mathcal{D}_2$, and
perform regression separately on both, resulting in two sets of fit parameters $\hbw_1$ and $\hbw_2$, respectively.
Using the results from the first training set, we calculate the training error, test error, and geometric test error.
We then evaluate the geometric bias and geometric variance using the prescription above, but only for the fixed pair of training data sets (and averaged over the test data set).
We repeat this process for many independent simulations and then average all error quantities.
The resulting geometric bias and geometric variance then approximate to those in Eq.~\eqref{eq:bv_approx}, averaged over all random quantities.

\subsection{Adversarial Perturbations}

To generate Fig.~\ref{fig:adversarial_perturb}(b), we used the same model used to create the schematic in Fig.~\ref{fig:adversarial_perturb}(a).
We then generated a random test data point $\vbx$, along with a collection of $200$ random pairs of adversarial and invariant perturbations.
To create each pair of perturbations, we randomly generated a vector $\vec{\mathbf{e}}$ drawn form the same distribution as $\vbx$and decomposed it into its components parallel and perpendicular to $\mathcal{F}_W$,
\begin{equation}
\begin{aligned}
\hat{\mathbf{e}}_\parallel &= \frac{\sum_i\vbf_{W, i}\vbf_{W, i}^T\vec{\mathbf{e}}}{\norm*{\sum_i\vbf_{W, i}\vbf_{W, i}^T\vec{\mathbf{e}}}}\\
\hat{\mathbf{e}}_\perp &= \frac{(I_{N_f}-\sum_i\vbf_{W, i}\vbf_{W, i}^T)\vec{\mathbf{e}}}{\norm*{(I_{N_f}-\sum_i\vbf_{W, i}\vbf_{W, i}^T)\vec{\mathbf{e}}}}.
\end{aligned}
\end{equation}

Next, we calculated the derivatives of the true and predicted labels for each perturbation using finite difference,
\begin{equation}
\begin{aligned}
\dv{y(\vbx + \eta\hat{\mathbf{e}})}{\eta} &\approx \frac{y(\vbx + \eta\hat{\mathbf{e}}) - y(\vbx)}{\eta}\\
\dv{\hat{y}(\vbx + \eta\hat{\mathbf{e}})}{\eta} &\approx \frac{\hat{y}(\vbx + \eta\hat{\mathbf{e}}) - \hat{y}(\vbx)}{\eta},
\end{aligned}
\end{equation}
with  $\eta = 10^{-2}$.

\subsection{Additional Numerical Results}

Here, we provide more detailed numerical results for the three models we consider in this work.
We show results for the model without basis functions in Fig.~\ref{fig:SI_lf}, with random linear features in Fig.~\ref{fig:SI_rlf}, and random nonlinear features in Fig.~\ref{fig:SI_rnlf}. 
For each model, we show the training error, test error, geometric test error, geometric bias, and geometric variance.
We also plots the Frobenius ($L_{2,2}$) norms of the orthogonal complements of the data projector $I_M - P_\ell$ and feature operator $I_{N_f} - P_f$, along with the minimum singular value of $Z$, $\sigma_{Z, \min}$ and maximum singular value of $P_f$, $\sigma_{\max}$.
For the model with basis functions, we also plot the subspace orientation angle $\theta_{\max}$ and projection deviation angle $\delta \phi_{\max}$ corresponding to the maximum singular value of $P_f$.
Results are shown as a function of $N_p/M = N_f/M$ for the model without basis functions and as a function of both $N_p/M$ and $N_f/M$ for the two models with basis functions. 
For each model, black dashed lines indicate phase transitions between different types of solutions~\cite{Rocks2020}.
Each point or pixel is averaged over at least 100 independent simulations.

\begin{figure*}
\centering
\includegraphics[width=1.0\linewidth]{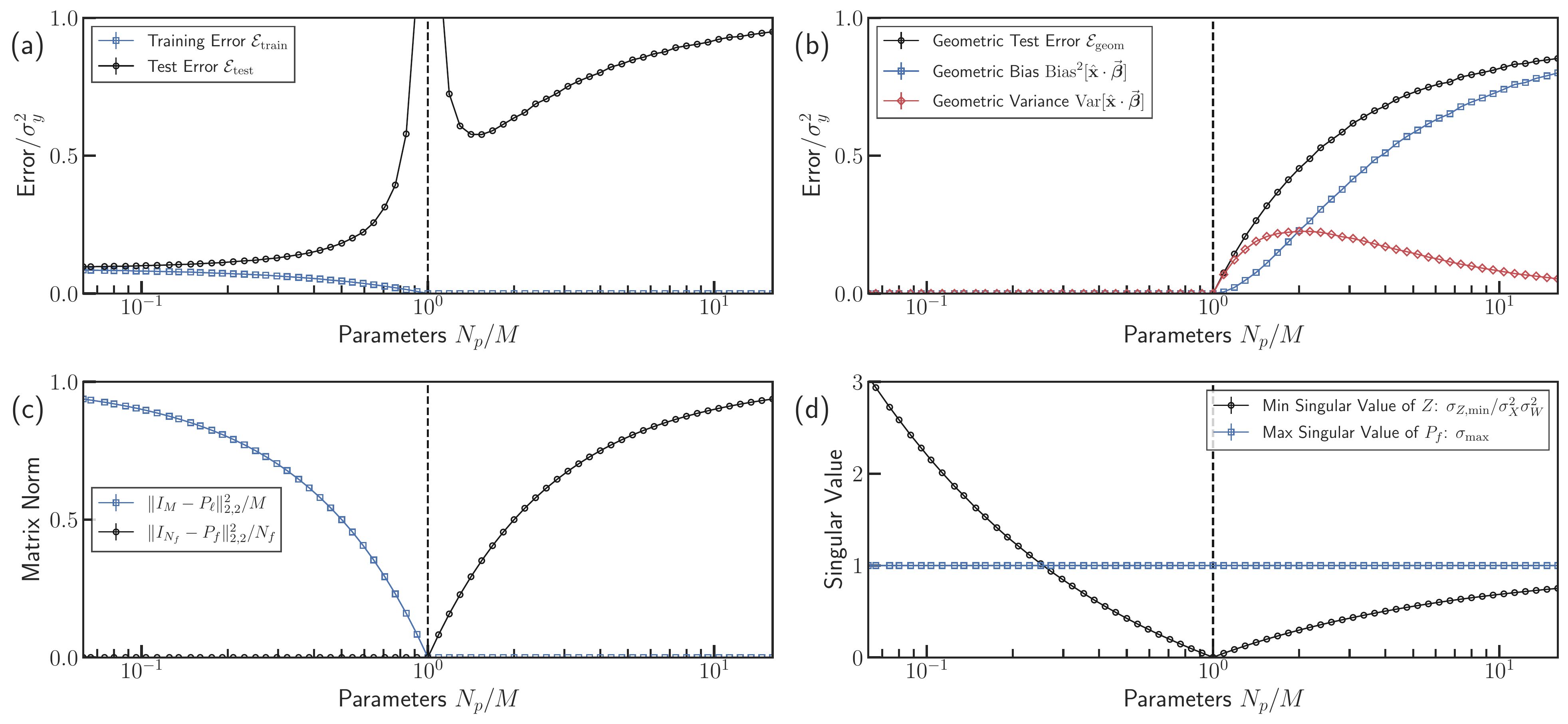}
\caption{{\bf No Basis Functions}. 
{\bf (a)} Average training error (blue squares) and  test error (black circles). 
{\bf (b)} Average geometric test error (black circles), geometric bias (blue squares), and geometric variance (red diamonds).
{\bf (c)} Average Frobenius matrix norm of the orthogonal complements of the data projector $I_M - P_\ell$ (blue squares) and feature operator $I_{N_f} - P_f$ (black circles).
{\bf (d)} Average minimum singular value of $Z$ (black circles) and maximum singular value of $P_f$ (blue squares).
Vertical dashed lines indicate the interpolation threshold.
 }
\label{fig:SI_lf}
\end{figure*}

\begin{figure*}
\centering
\includegraphics[width=1.0\linewidth]{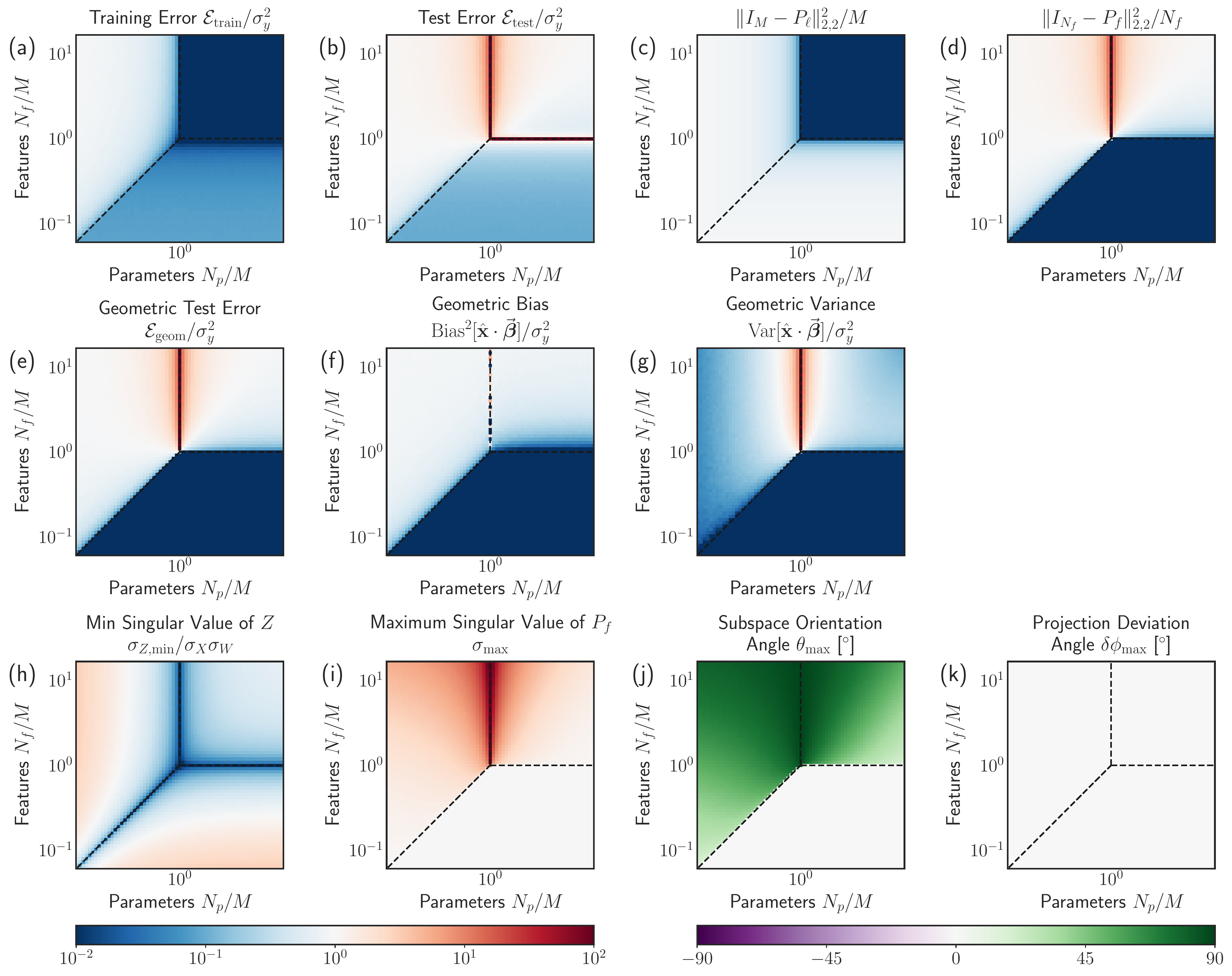}
\caption{{\bf Random Linear Features}. 
Average {\bf (a)} training error, {\bf (b)} test error,
{\bf (e)} geometric test error, {\bf (f)} geometric bias, and {\bf (g)} geometric variance.
Average Frobenius matrix norm of the orthogonal complements of {\bf (c)} the data projector $I_M - P_\ell$ and {\bf (d)} the feature operator $I_{N_f} - P_f$ .
Average {\bf (h)} minimum singular value of $Z$ and  {\bf (i)} maximum singular value of $P_f$.
Average {\bf (j)} subspace orientation angle $\theta_{\max}$ and  {\bf (k)}  projection deviation angle $\delta \phi_{\max}$ associated with the maximum singular value of $P_f$. The angle $\delta \phi_{\max}$ is taken to be zero when $\theta_{\max} = 0$ and $\sigma_{\max} = 1$.
Vertical and horizontal lines indicate the interpolation threshold.
The diagonal dashed lines indicate the boundary between the geometrically biased and unbiased regimes.
 }
\label{fig:SI_rlf}
\end{figure*}

\begin{figure*}
\centering
\includegraphics[width=1.0\linewidth]{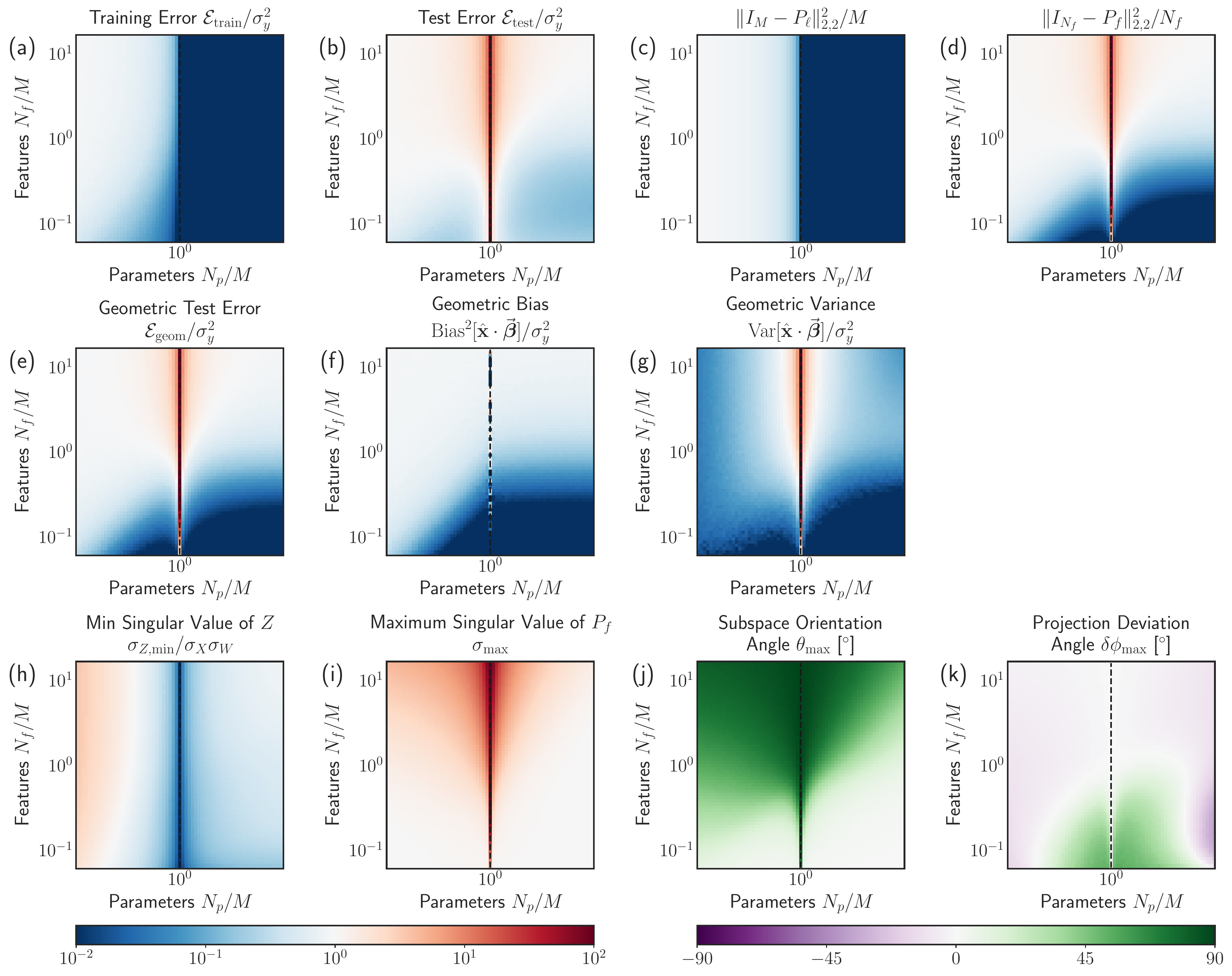}
\caption{{\bf Random Nonlinear Features}. 
Average {\bf (a)} training error, {\bf (b)} test error,
{\bf (e)} geometric test error, {\bf (f)} geometric bias, and {\bf (g)} geometric variance.
Average Frobenius matrix norm of the orthogonal complements of {\bf (c)} the data projector $I_M - P_\ell$ and {\bf (d)} the feature operator $I_{N_f} - P_f$.
Average {\bf (h)} minimum singular value of $Z$ and  {\bf (i)} maximum singular value of $P_f$.
Average {\bf (j)} subspace orientation angle $\theta_{\max}$ and {\bf (k)} projection deviation angle $\delta \phi_{\max}$ associated with the maximum singular value of $P_f$.
Vertical lines indicate the interpolation threshold.
 }
\label{fig:SI_rnlf}
\end{figure*}

\end{document}